%% file: root.tex
\documentclass[letterpaper,10pt,conference]{ieeeconf}

\usepackage{multicol}
\usepackage[bookmarks=true]{hyperref}
\IEEEoverridecommandlockouts   
\input{preamble.tex}

\title{\LARGE \bf Autonomy Architectures for Safe Planning in Unknown Environments Under Budget Constraints}

\ifx\anonymized\undefined 
\author{Daniel M. Cherenson, Devansh R. Agrawal, and Dimitra Panagou 
\thanks{This research was supported by the Center for Autonomous Air Mobility and Sensing (CAAMS), an NSF IUCRC, under Award Number 2137195, and an NSF CAREER under Award Number 1942907.}
\thanks{Daniel Cherenson is with the Department of Robotics, University of Michigan, Ann Arbor, MI 48109 USA {\tt\footnotesize dmrc@umich.edu}}
\thanks{Devansh Agrawal is with the Department of Aerospace Engineering, University of Michigan, Ann Arbor, MI 48109 USA {\tt\footnotesize devansh@umich.edu}}
\thanks{Dimitra Panagou is with the Department of Robotics and Department of Aerospace Engineering, University of Michigan, Ann Arbor, MI 48109 USA {\tt\footnotesize dpanagou@umich.edu}}%
\thanks{$^{*}$Correspondence: {\tt\small dmrc@umich.edu} }
}
\else
\author{Author Names Omitted for Anonymous Review. Paper-ID}
\thanks{$^{*}$[redacted]}

\fi
\begin{document}

\maketitle

\begin{abstract}
Mission planning can often be formulated as a constrained control problem under multiple path constraints (i.e., safety constraints) and budget constraints (i.e., resource expenditure constraints). In \textit{a priori} unknown environments, verifying that an offline solution will satisfy the constraints for all time can be difficult, if not impossible. We present \reroot{}, a novel sampling-based framework that enforces safety and budget constraints for nonlinear systems in unknown environments. The main idea is that \reroot{} grows multiple reverse RRT* trees online, starting from renewal sets, i.e., sets where the budget constraints are renewed. The dynamically feasible backup trajectories guarantee safety and reduce resource expenditure, which provides a principled backup policy when integrated into the \gatekeeper{} safety verification architecture. We demonstrate our approach in simulation with a fixed-wing UAV in a GNSS-denied environment with a budget constraint on localization error that can be renewed at visual landmarks. [\href{https://github.com/dcherenson/budget-constrained-planning}{Code}]
\end{abstract}

\section{Introduction}
\label{sec:introduction}

Robotic and autonomous systems often operate with limited knowledge of their environment and limited sensing capabilities, which poses challenges to guaranteeing their safe operation. Safety requirements are typically given in the form of a set of states in which the robot must always remain, called the safe set. In unknown environments, the safe set is not fully known \textit{a priori} and must be built online with information from the sensor measurements. 

In addition to instantaneous safety requirements, many robotic systems have a limited budget on resources or quantities that are depleted throughout a mission. Constraints on finite resources have been studied under various names, including integral constraints~\cite{kumar2010efficient}, budget constraints~\cite{takei2015optimal,tsiogkas2018dcop}, and cost constraints~\cite{yang2021uav}. Previously studied examples include constraints on battery charge~\cite{naveed2024eclares}, localization error~\cite{bopardikar2014multi}, or time spent visible to an enemy observer~\cite{gilles2020evasive}. We consider a generalized case of the budget-constrained, safety-critical planning problem where resources can be renewed at specific regions in the state space.

Various offline methods to enforce budget constraints have been introduced. The methods in~\cite{kumar2010efficient, takei2015optimal} solve trajectory-optimization problems with renewable budget constraints. However, these approaches rely on searching over large-scale graphs or solving partial differential equations derived from dynamic programming, which is computationally expensive and better suited for offline planning in known environments. An offline reinforcement-learning approach is studied in  \cite{lin2023safe}, where budget constraints for high-dimensional, nonlinear systems are incorporated in the training process, however no strict guarantees of constraint satisfaction are obtained.

In this paper, we instead focus on the problem of online safety and budget constraint satisfaction, where lack of knowledge of the environment requires constant replanning and infinite horizon guarantees. Previous work on online budget constraint satisfaction used simplified robot and budget dynamics~\cite{notomista2018persistification,yang2021uav}. The solutions were designed for specific use cases, limiting their general applicability. 

Our approach relies on a run-time safety filter. Many safety-critical robotic systems apply these filters with a backup or recovery policy that reaches an invariant backup set to guarantee safety over an infinite horizon~\cite{hobbs2023runtime}. For systems with instantaneous safety constraints, finding a backup set is often trivial, as most robots can stop in place or return to a known safe configuration. Typical backup policies include braking maneuvers, model predictive control approaches, and learned RL policies~\cite{kim2021backup,thananjeyan2021recovery,kiemel2024safe,jung2025contigency}. However, to guarantee budget constraint satisfaction, the design of backup policies is more challenging because the budget must be renewed in the backup set, which may only occur in limited subsets of the state space. Additionally, the trajectories to reach the backup sets must satisfy all safety constraints, which forms the backup trajectory generation problem.

To the best of the authors' knowledge, there are no solutions that explicitly guarantee safety and budget constraint satisfaction with limited environmental knowledge and general, non-trivial budget expenditure dynamics. This paper aims to narrow this gap by proposing a recursively feasible algorithm for the online trajectory generation of nonlinear systems under budget and safety constraints in unknown environments. The key idea and contribution is the sampling-based method called \reroot{}, which efficiently generates backup trajectories to minimize resource expenditure while ensuring that the system reaches budget-renewal sets, where resources are replenished. The algorithm is demonstrated on a UAV that uses visual odometry to reach a goal in an unknown, GNSS-denied environment, with guaranteed satisfaction of budget and safety constraints.

The paper is organized as follows. In Section \ref{section:problem_formulation}, we introduce the problem formulation. Section \ref{section:method} describes the proposed method, highlighting the safe trajectory planner and \reroot{}. Section \ref{section:experiments} discusses a simulation case study of a UAV in a GNSS-denied environment navigating to a goal with safety and budget constraints related to visual odometry.

\section{Problem Formulation}
\label{section:problem_formulation}
We represent the robot motion by the nonlinear dynamics:
\eqn{\dot{x} = f(x,u), \label{eq:dynamics}} 
where $x \in \Xcal \subset \R^n$ is the state and $u \in \Ucal \subset \R^m$ is the control input, and $f : \Xcal \times \Ucal \to \R^n$ a locally Lipschitz function. 
Under a (locally Lipschitz) feedback controller $\pi : \Xcal \to \Ucal$ and an initial condition $x(t_0) = x_0 \in \Xcal$, the closed-loop system dynamics read: \eqn{\dot{x} = f(x,\pi(x)), ~x(t_0)=x_0.\label{eq:closed_loop}}

We consider a resource that is subject to a budget constraint, and which can be renewed in a subset $\Rcal$ of the state space $\Xcal$. The budget state $b$ is governed by the following hybrid dynamical system: \begin{equation}
\begin{aligned}
\begin{cases}
\dot{b} = L(x,u), &\quad x \notin \Rcal \\
b^+ = b_{\textup{reset}}, &\quad x \in \Rcal
\end{cases}
&\qquad
b(t_0) = b_0, 
\end{aligned}
\label{eq:budget_hybrid}
\end{equation}
 where $L : \Xcal \times \Ucal \to \Rnonneg$ is piecewise continuous and $b^+ = b_{\textup{reset}} \ge 0$ is the reset value of the budget state after the jump. We assume that $L$ is known, but $\Rcal$ is \textit{not} fully known.
We aim to satisfy the following constraints:

\begin{equation}
    x(t) \in \Scal \land b(t) \le B, \quad \forall t \ge t_0,
    \label{eq:cons}
\end{equation}

where $\Scal \subset \Xcal$ is the set of states that satisfy safety requirements and $B$ is the budget constraint that must not be exceeded. We assume that $B > b_{\textup{reset}}$, that $\Scal$ has a nonempty interior, and that it is not fully known, but can be sensed online from sensor measurements. We introduce the known subset of the safe set at time $t_k$, $\Fcal_k \subset \Scal$, where $k\in\naturals$. Similarly, we denote the known renewal set as $\Rcal_k \subset \Rcal \subset \Scal$, where we have assumed the renewal set is safe.

At time $t_{k+1}$, new information is acquired by the robot to update $\Fcal_k$ and $\Rcal_k$ such that the following relations hold: \eqn{\Fcal_k \subset \Fcal_{k+1},~\Rcal_k \subset \Rcal_{k+1}, \quad \forall k \in \naturals 
.}

Moreover, $\Rcal_k$ can be composed of $N_C\ge1$ disjoint, compact subsets. At time $t_k$, we denote the $i$-th subset as $\Rcal^i_k$ for $i \in \{1,\dotsc,N_C\}$ such that $\Rcal_k =\cup_{i=1}^{N_C}\Rcal^i_k$.

Our method relies on concepts in set invariance to guarantee constraint satisfaction \cite{blanchini1999set}. 
\begin{definition}[Controlled Invariant Set]
A set $\Ccal \subset \Xcal$ is controlled invariant for the system \eqref{eq:dynamics}, if there exists a controller $\pi : \Xcal \to \Ucal$ that assures the existence of a unique solution to the closed-loop system \eqref{eq:closed_loop} such that $\Ccal$ is positively invariant for the closed-loop system \eqref{eq:closed_loop}, i.e., $\forall x(t_0) \in \Ccal,  x(t) \in \Ccal ,\forall t \ge t_0.$
\end{definition}

\begin{definition}[Backup Set]  \label{def:backup_set}
    A set $\Bcal \subset \Xcal$ is a \textbf{backup set} if there exists a controller $\pi^B : \Xcal \to \Ucal$ such that for the closed-loop system \eqref{eq:closed_loop}, $\forall~x(t_0)$ in a neighborhood $\Ncal \subset \Xcal$, there exists a finite time $0\leq T_B(x(t_0))< \infty$ such that $\Bcal$ is reachable at $t_0+T_B$ and controlled invariant thereafter, i.e., $x(t) \in \Bcal,~ \forall t \ge t_0+T_B$.
\end{definition}

\begin{assumption}\label{assumption:budget_reset} We assume that the budget renewal set $\Rcal_k$ is a backup set for the closed loop system \eqref{eq:closed_loop} for all $k\in\naturals$. Then, from \eqref{eq:budget_hybrid}, the budget constraint is satisfied for all time in $\Rcal_k$: \eqn{x(t) \in \Rcal_k \implies b(t) = b_{\textup{reset}} \le B, \quad \forall t\ge t_k.} 
\end{assumption}
\begin{definition}[Trajectory]
    A \textbf{trajectory} with horizon $T_H$ is a pair of functions $p : \Tcal \to \Xcal$ and $u : \Tcal \to \Ucal$ defined on $\Tcal = [t_0,t_0+T_H]\subset \Rnonneg$ that satisfy \eqn{\dot p(t) = f(p(t),u(t)) \quad \forall t \in \Tcal.} The set of all trajectories at $t \in \R$ starting from $x \in \Xcal$ is \eqn{\Phi(t,x) = \{ (p,u) ~|~p(t) = x \land (p,u) \text{ is a trajectory}\}.}
\end{definition}
We assume there exists a nominal planner to generate trajectories that aim to fulfill the mission objectives. 
\begin{definition}[Nominal Trajectory]
    
At time $t_k \in \R$ at state $x_k \in \Xcal$, the planner generates the \textbf{nominal trajectory} $(p_k^{\textup{nom}} , u_k^{\rm nom}) \in \Phi(t_k,x_k)$ defined over the interval $\Tcal =[t_k,t_k+T_H]$.
\end{definition}

\begin{figure*}
    \centering
        \begin{minipage}[t]{0.4\linewidth}
        \centering
        \includegraphics[width=\linewidth]{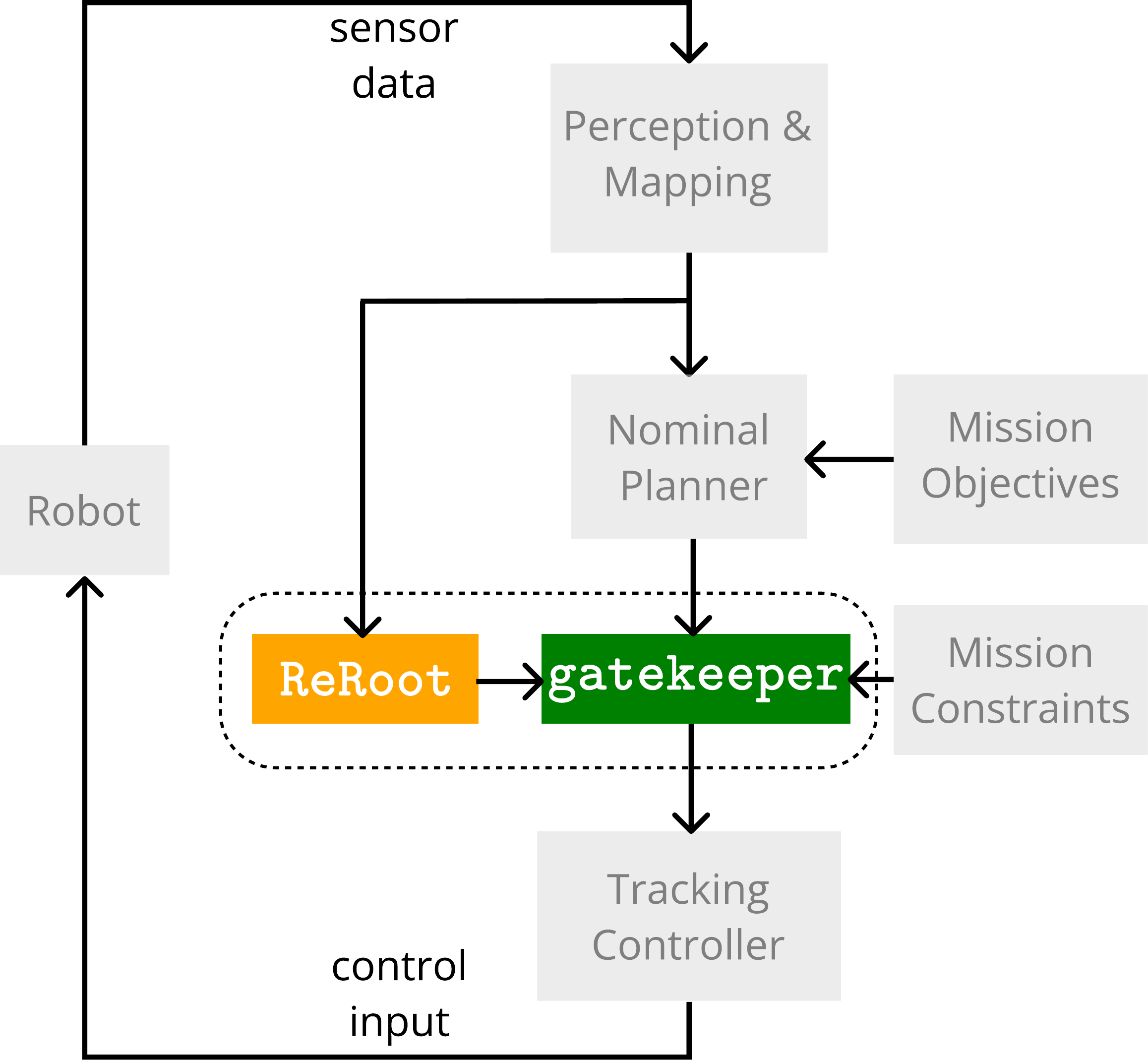}
        \caption{Block diagram of our proposed layered autonomy architecture, with our module \gatekeeper{} + \reroot{} highlighted in the dashed line.
        }
    \label{fig:block_diagram}
    \end{minipage}%
    \hfill 
    \begin{minipage}[t]{0.57\linewidth}
        \centering
        \includegraphics[width=\linewidth]{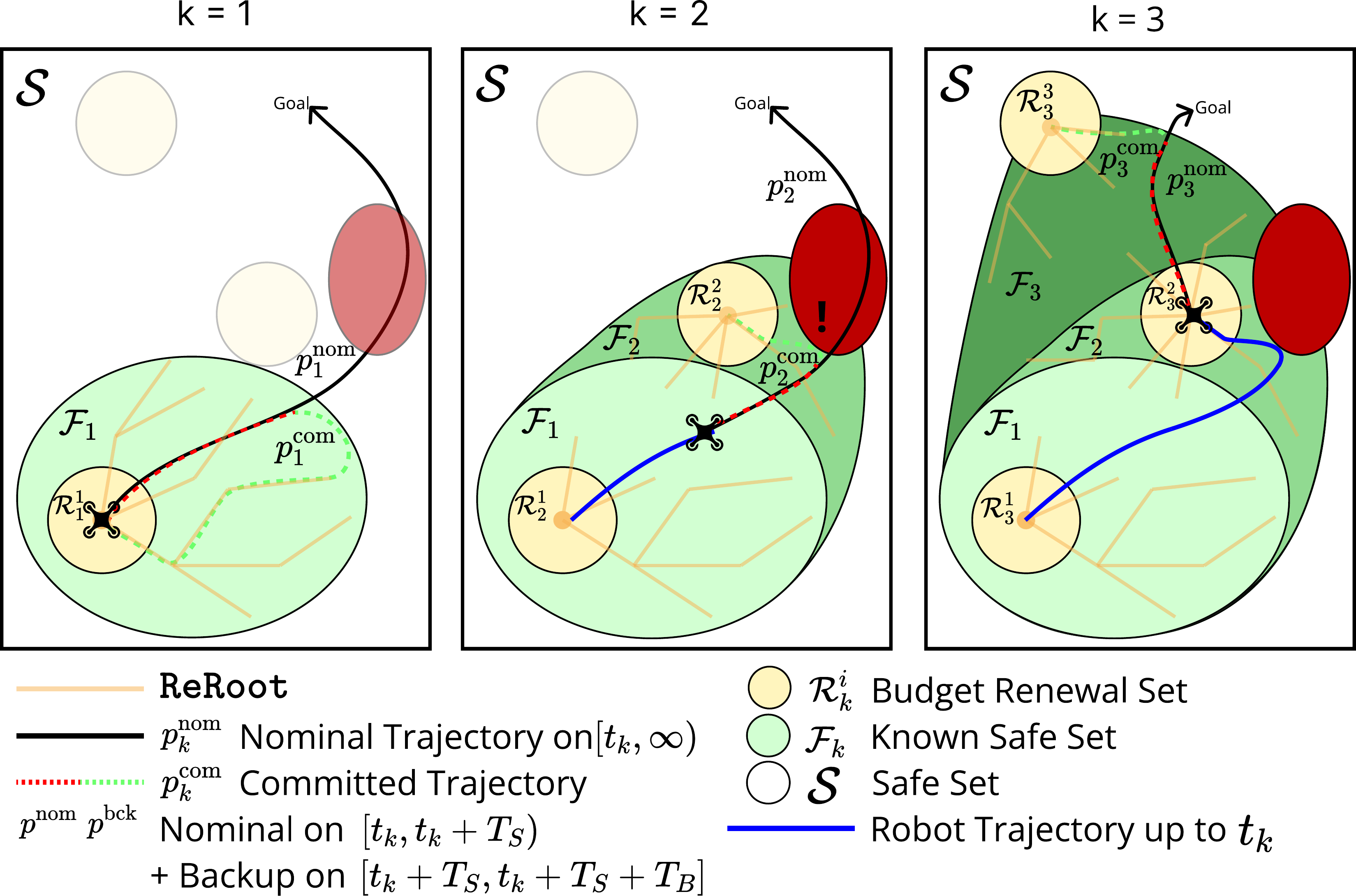}
        \caption{Snapshots of \gatekeeper{} with \reroot{} trees rooted at budget renewal sets. The trees are grown in the free set $\Fcal_k$. The robot discovers new budget renewal sets as it uncovers more of the unknown space.}
        \label{fig:backup_rrt}
    \end{minipage}   
\end{figure*}
With the above preliminaries, we state the main problem:
\begin{problem}
\label{problem:main_prob}
    Given the known safe set $\Fcal_k$ and budget renewal set $\Rcal_k$ updated at each $k\in\naturals$, design a safe planning algorithm that recursively constructs a trajectory $p(t)$ that satisfies all safety and budget constraints, i.e., $p(t) \in \Scal$ and $b(t) \le B$ for all $t \ge t_0$.
\end{problem}

\section{Method Overview}
\label{section:method}
To address Problem \ref{problem:main_prob}, we propose a framework that recursively generates trajectories that are guaranteed to satisfy the safety and budget constraints over an infinite horizon. The solution has two components: A)~\gatekeeper{}, an algorithm to construct and validate safe trajectories using the nominal trajectory and a backup trajectory, and B)~\reroot{}, a sampling-based algorithm that allows for efficient online generation of backup trajectories. See \Cref{fig:block_diagram} for a block diagram showing where the \gatekeeper{} and \reroot{} components fit into the general autonomy stack.

\subsection{Guaranteed Constraint Satisfaction}

Here, we list the key steps of our method, and then we provide mathematical definitions of each type of trajectory along with a proof of correctness. We use the \gatekeeper{} framework~\cite{agrawal2024gatekeeper,agrawal2025online}, which recursively filters the output of the nominal planner by constructing a trajectory that is guaranteed to remain in the safe set for all time. We make use of the reformulated budget dynamics in \eqref{eq:budget_hybrid} in the \gatekeeper{} framework to augment the definition of the safe set to include the budget constraint. \Cref{alg:gk} details the $k$-th iteration, where $k\in\naturals\setminus0$. 
\begin{algorithm}[h]
\small
    \DontPrintSemicolon
\caption{\gatekeeper{} with \reroot}
    \label{alg:gk}

$T_S \gets T_H$\;
\While{$T_S \ge 0$}{
    $p_k^{\textup{can},T_S}, b^{\textup{nom}} \gets \FuncSty{propagate}(p^{\textup{nom}}_k, x(t_k), [t_k, t_k + T_S])$\;

    $p^{\textup{bck,T_S}}_k, b^{bck} \gets \FuncSty{GetBackupFromReRoot}(\Gcal_k, p_k^{\textup{can},T_S}(t_k+T_S),\Fcal_k)$\; \label{line:call_reroot}

    $p^{\textup{can},T_S}_k \gets \FuncSty{append}(p^{\textup{can},T_S}_k, p^{\textup{bck,T_S}}_k)$\; 
    $b^{-} \gets b^{\textup{nom}} + b^{\textup{bck}}$\;
    \If{$p_k^{\textup{can},T_S}$ is valid by \Cref{def:valid}}{
        $p_k^{\textup{com}} = p_k^{\textup{can},T_S}$\;
    \Return{$p^{\textup{com}}_k$}\;
    }
    $T_S \gets T_S - \Delta T$\;
}
$p_k^{\textup{com}} = p_{k-1}^{\textup{com}} $\;
\Return{$p^{\textup{com}}_k$}\;
\end{algorithm}

We also summarize the steps below:
\begin{itemize}
    \item For all $T_S \in [0,T_H]$, propagate the nominal trajectory $p_k^{\textup{nom}}$ (black line in \Cref{fig:backup_rrt}) on the interval $[t_k,t_k + T_S]$, then generate a backup trajectory $p^{\textup{bck},T_S}_k$ that reaches $\Rcal_k$ with \reroot{}, which forms a set of \textit{candidate trajectories}.
    \item For each candidate trajectory, compute the budget state time history $b(t)$ with \eqref{eq:budget_hybrid} and verify that constraints \eqref{eq:cons} are satisfied, forming a set of \textit{valid trajectories}. In \Cref{fig:backup_rrt} when $k=2$, candidate trajectories that enter the red unsafe area are deemed invalid.
    \item Select the valid trajectory that maximizes $T_S$ as the \textit{committed trajectory} $p_k^{\textup{com}}$. In \Cref{fig:backup_rrt}, the committed trajectory is composed of a red nominal portion and a green backup portion, which is generated from \reroot{} and will be detailed in the following subsections. If there are no valid trajectories, select the $(k-1)$-th committed trajectory $p^{\textup{com}}_{k-1}$. This step guarantees recursive feasibility.
\end{itemize}

Overall, \Cref{fig:backup_rrt} depicts the \gatekeeper{} trajectories and \reroot{} growth as the robot expands the known budget renewal subsets $\Rcal_k^i$ during the mission.

To formally prove that all committed trajectories are guaranteed to satisfy all constraints for all time, we need to introduce the notion of a backup trajectory.
\begin{definition}[Backup Trajectory]
    A trajectory $(p^{\textup{bck},T_S}_k, u^{\textup{bck},T_S}_k) \in \Phi(t_k + T_S,x(t_k+T_S))$ over the interval $\Tcal = [t_k+T_S,\infty)$ is a \textbf{backup trajectory} to a set $\Bcal \subset \Xcal$ if the following two conditions hold: \eqn{p^{\textup{bck},T_S}_k(t_k+T_S+T_B) \in \Bcal,} \eqn{u^{\textup{bck},T_S}_k(t) = \pi^B(p^{\textup{bck},T_S}_k(t)) \quad \forall t \ge t_k+T_S+T_B.}
\end{definition}
The backup trajectory is dynamically feasible and enters a backup set within $T_B$ seconds after the switching time $T_S$. After entering $\Bcal$, the trajectory remains in $\Bcal$ for all time because the control policy $\pi^B$ renders $\Bcal$ controlled-invariant.

Note that the backup trajectory is not required to satisfy constraints \eqref{eq:cons} -- the validation step of the candidate trajectories will filter out unsafe backup trajectories. In general, finding a backup trajectory from any initial condition to a backup set is a key step in the \gatekeeper{} framework, which is covered in the next subsection on \reroot{}.

A candidate trajectory follows the nominal trajectory then switches to a backup trajectory.
\begin{definition}[Candidate Trajectory]
\label{def:can}
    At time $t_k \in \R$ and starting from $x_k \in \Xcal$, let the nominal trajectory be $(p_k^{\textup{nom}} , u_k^{\rm nom}) \in \Phi(t_k,x_k)$. Let $t_{kS} = t_k+T_S$ and $\Tcal_S = [t_k,t_{kS})$. For any switching time $T_S \in [0,T_H]$, the backup trajectory is $(p_k^{\textup{nom}}, u_k^{\rm nom}) \in \Phi(t_{kS},p_k^{\rm nom}(t_{kS}))$. A \textbf{candidate trajectory} $(p_k^{\textup{can}} , u_k^{\rm can}) \in \Phi(t_k,x_k)$ with switching time $T_S$ is defined as \eqn{(p_k^{\textup{can}}(t), u_k^{\rm can}(t)) = \begin{cases} (p_k^{\textup{nom}}(t) , u_k^{\rm nom}(t))~ & \text{if}~ t \in \Tcal_S,
        \\ (p_k^{\textup{bck},T_S}(t) , u_k^{\textup{bck},T_S}(t)) ~ &\text{if} ~t \ge t_{kS}.
    \end{cases}}

\end{definition}

A candidate trajectory is valid if it satisfies all constraints at every point along the trajectory and reaches a budget renewal set in finite time. 

\begin{definition}[Valid Trajectory]
    \label{def:valid}
    A candidate trajectory is \textbf{valid} if it remains in the \textit{known} safe set:
    \eqn{p_k^{\textup{can},T_S}(t)\in\Fcal_k \quad \forall t \in [t_k,t_{kB}],} the solution to \eqref{eq:budget_hybrid} remains below the budget $B$: \eqn{b(t) \le B \quad \forall t \in [t_k,t_{kB}],} and the trajectory reaches $\Rcal_k$: \eqn{p_k^{\textup{can},T_S}(t_{kB}) \in \Rcal_k.}
\end{definition}
Finally, the committed trajectory is chosen to be the valid trajectory that maximizes $T_S$. If no committed trajectory can be found, the previously committed trajectory is followed.
\begin{definition}[Committed Trajectory]
\label{def:com}
    At iteration $k$, the set of valid candidate trajectories parameterized by $T_S$ is \eqn{\Ical_k = \{T_S\in[0,T_H]~|~ (p_k^{\textup{can},T_S},u_k^{\textup{can},T_S}) \textup{ is valid (Def. \ref{def:valid})} \}.}    
    If $\Ical_k \ne \emptyset$, let $T_S^* = \max \Ical_k$. The \textbf{committed trajectory} is \eqn{p_k^{\textup{com}}(t) = p_k^{\textup{can},T_S^*}(t), \quad t \in [t_k,\infty).}
    If $\Ical_k = \emptyset$, the \textbf{committed trajectory} is \eqn{p_k^{\textup{com}}(t) = p_{k-1}^{\textup{com}}(t), \quad t \in [t_k,\infty).}
\end{definition}
{The following assumption is required for the inductive recursive feasibility property of \gatekeeper{}.
\begin{assumption}
    We are given an initial dynamically feasible valid candidate trajectory $p_0^{\textup{can},T_S}$ which by definition satisfies all safety and budget constraints.\footnote{This assumption is not restrictive to satisfy in practice. The initial candidate trajectory can be arbitrarily short and planned within the known safe set, as long as it ends in a backup set. For example, in our case study in \Cref{section:experiments}, the initial candidate trajectory is a short out-and-back loop that stays within the initial field of view of the UAV and returns to the controlled-invariant orbit.}
\end{assumption}

Now we can state the following theorem which shows that \gatekeeper{} guarantees constraint satisfaction for all time.
\begin{theorem}
\label{theorem}
    Suppose $p_0^{\textup{can},T_S}$ is a dynamically feasible candidate trajectory defined on $[t_0,\infty)$ that is valid by Def.~\ref{def:valid} for some $T_S\ge0$. If $p_k^{\textup{com}}$ is determined via Def.~\ref{def:com}, then \eqnN{p_k^{\textup{com}}(t)\in\Scal \land b(t) \le B \quad \forall t\in [t_k,\infty)
    .}

\end{theorem}
\begin{proof}
    The proof by induction is based on Theorem 1 in~\cite{agrawal2024gatekeeper}.
    
    \textit{Base Case:} Let $k$ = 0. Since $p^{\textup{can},T_S}_0$ is valid, it is committed as $p_0^{\textup{com}} = p^{\textup{can},T_S}_0$. Then for the committed trajectory, \eqnN{&p_0^{\textup{com}}(t) \in \begin{cases}        \Fcal_0, \quad t \in [t_0, t_{0,B}) \\ \Rcal_0, \quad t \ge t_{0,B}  \end{cases} \\ \implies &p_0^{\textup{com}}(t)  \in \begin{cases} \Scal, \quad t \in [t_0, t_{0,B}) \\ \Scal, \quad t \ge t_{0,B}  \end{cases} \\ \iff &p_0^{\textup{com}}(t)  \in \Scal,\quad \forall t\ge t_0.} where $t_{0,B} = t_0+T_S+T_B$. We used the fact that $\Rcal_0$ is a backup set by \Cref{def:backup_set} and is thus controlled invariant. Hence, $p_0^{\textup{com}}$ satisfies all constraints for all time. 
    Then for the budget state solution to \eqref{eq:budget_hybrid} with $p(t)$ and the corresponding $u(t)$, \eqnN{b(t) \begin{cases}
        \le B, \quad t \in [t_0,t_{0B}) \\ = b_{\textup{reset}} \le B, \quad t\ge t_{0B}
    \end{cases}  \iff b(t) \le B, \; \forall t \ge t_0.} 
    
    \textit{Induction Step:} Now suppose the claim in \Cref{theorem} is true for some $k\in\naturals$. To show the claim holds for $k+1$, consider the two possible definitions for $p_{k+1}^{\textup{com}}$ from \Cref{def:com}:
    
    \textit{Case 1:} If $\Ical_{k+1} \ne \emptyset$, then \eqnN{p_{k+1}^{\textup{com}}(t) &= p_{k+1}^{\textup{can},T_S^*}(t) ~\forall t \ge t_{k+1} \\ 
    &\in  \begin{cases}
        \Fcal_{k+1}, \quad t \in [t_{k+1},t_{k+1,SB}) \\ \Rcal_{k+1}, \quad t \ge t_{k+1,SB}     \end{cases} \\ 
        & \in \Scal, \quad \forall t \ge t_{k+1}.
}

    \textit{Case 2:} If $\Ical_{k+1}=\emptyset$, the committed trajectory is \eqnN{p_{k+1}^{\textup{com}}(t) = p_{k}^{\textup{com}}(t) \in \Scal, \quad \forall t \ge t_{k+1}.}
    Therefore, $p_k^{\textup{com}}(t) \in \Scal ~\forall t \in [t_k,t_{k+1}), ~\forall k \in \naturals \implies p(t) \in \Scal~\forall t \ge t_0$. Moreover, the solution to \eqref{eq:budget_hybrid} with the corresponding $p(t)$ and $u(t)$ satisfies $b(t) \le B~\forall t \ge t_0$.
\end{proof}

We have shown the \gatekeeper{} framework guarantees constraint satisfaction when paired with a suitable backup policy that ensures all candidate trajectories reach $\Rcal_k$.
An open challenge is how to construct a general backup policy that applies to a wide range of nonlinear system dynamics and non-convex safety constraints. To address this challenge, we introduce an algorithm to efficiently construct backup trajectories as the robot gathers information in the environment.}

\subsection{\reroot{} for Backup Trajectory Construction}
{We leverage the well-known RRT* path planner in a novel manner to propose a sampling-based algorithm called \reroot{} that efficiently generates a dynamically feasible trajectory from any location to a backup set in which the budgeted resource is renewed. 

\subsubsection{Backup Trajectory Generation Problem}

Consider a valid candidate trajectory $p(t)$ and the corresponding budget state prediction $b(t)$. Let the pre-jump budget state be the one-sided limit as $t$ approaches $t_{kB}$ from below: \eqn{b^- = \lim_{t\nearrow t_{kB}}b(t),} where $t_{kB} = t_k+T_S+T_B$ is the finite time when $p$ reaches $\Rcal_k$. For any switch time $T_S$, the ideal backup trajectory is the one that minimizes the budget expenditure. The goal is to find such a backup trajectory to form the \gatekeeper{} candidate trajectory.

\begin{problem}
\label{prob:backup_prob}
    At the $k$-th iteration and for a given switch time $T_S$, given the known budget renewal set $\Rcal_k$ at time $t_k$, first return a set of dynamically feasible trajectories starting at $p_k^{\textup{can},T_S}(t_{kS})$ and reaching $\Rcal_k$ while remaining in the known safe set $\Fcal_k$. Select among them the trajectory $p^{\textup{bck},T_S}_k$ with minimal budget expenditure $b^-$.
\end{problem}
\begin{algorithm}[t]
\small
    \DontPrintSemicolon
    \caption{Main Autonomy Loop}
    \label{alg:autonomy}
    $\Gcal_k = \FuncSty{InitReRoot}(\Rcal_0)$\; \label{line:init_reroot}
    \For{$k \in [1,\dotsc, K]$}{
        $\Fcal_k, \Rcal_k \gets \FuncSty{UpdateFreeSpace}()$\;
        \If{$\Rcal_k \cap \Rcal_{k-1} \ne \emptyset$}{
            $\FuncSty{AddRootNodes}(\Gcal_k, \Rcal_k \cap \Rcal_{k-1})$\; \label{line:add_roots} 
        }
        $\FuncSty{GrowReRoot}(\Gcal_k, \Rcal_k,\Fcal_k, n_{\textup{update}})$\; \label{line:grow_reroot}
        $p_k^{\textup{com}} = \FuncSty{gatekeeper}(p_k^{\textup{nom}}, \Gcal_k, x(t_k),b(t_k))$\; \label{line:gk}
        Track $p_k^{\textup{com}}$\;
        Update $b(t)$ via \eqref{eq:budget_hybrid}\;
    }
\end{algorithm}

\subsubsection{\reroot{} (\underline{Re}verse \underline{Root}ed Forest)}

We build a graph $\Gcal_k = (\Vcal_k,\Ecal_k)$ with a set of nodes $\Vcal_k \subset \Xcal$ and a set of edges $\Ecal_k \subset \Vcal_k \times \Vcal_k$. The underlying algorithm for building $\Gcal_k$ is either the standard RRT* algorithm~\cite{karaman2011sampling}, or extensions of RRT* that attempt to generate kinodynamically feasible paths~\cite{perez2012lqr,webb2013kinodynamic}. The novelty in our approach is the manner of initialization and interpretation of the generated graph $\Gcal_k$. \Cref{alg:autonomy} details the overall autonomy algorithm that runs at each iteration $k$, including the novel \reroot{} initialization and growth steps. We summarize the key steps of \reroot{}:

\begin{itemize}
    \item \textbf{Root Nodes:} The graph $\Gcal_k$ is a rooted forest, i.e., a union of disjoint rooted trees where all root nodes are in the renewal set $\Rcal_k$. The initialization step, \Cref{line:init_reroot}, creates the set of root nodes $\Vcal_0$ in $\Rcal_0$. During the mission, further updates via \Cref{line:add_roots} add more root nodes to $\Vcal_k$ in the expanded renewal set $\Rcal_k$. The root nodes serve as the endpoints of the backup trajectories.

    \item \textbf{Online Growth:} In \Cref{line:grow_reroot} at each iteration $k$, the RRT* algorithm is used to add $n_{\rm update}$ nodes to $\Gcal_k$, including the standard rewiring step. The path cost function used in RRT* is the budget dynamics \eqref{eq:budget_hybrid} evaluated along the path. During the rewiring step, a node can switch trees, meaning that the path to the root of the new tree has a lower cost than the node's path to its previous root node. 

    \item \textbf{Path of Waypoints:} From any node $v \in \Vcal_k$ in the forest, there is a single path of waypoints $\Wcal(v)$ to a root node in $\Rcal_k$, which is not necessarily dynamically feasible. This path is used to generate a dynamically feasible trajectory by forward propagation of the closed-loop dynamics, which we describe below. This step occurs in \Cref{line:call_reroot} of Alg.~\ref{alg:gk}, called at each switching time $T_S$.
   \end{itemize}
    
    \begin{remark} Some popular variants of RRT/RRT* involve growing two trees with one root at the initial location and one at the goal location \cite{kuffner2000rrt,jordan2013optimal,qureshi2015intelligent}. By contrast, \reroot{} grows multiple trees by placing new roots as more components of the renewal set $\Rcal_k$ are discovered online. 
    \end{remark}

Next, we describe the in further detail the process of extracting a backup trajectory from the \reroot{} graph during the \gatekeeper{} candidate trajectory propagation step~(\Cref{line:call_reroot} of \Cref{alg:gk}).
\subsubsection{Backup Trajectory Generation}
Denote $d : \Vcal_k \to \Rnonneg$ as the depth of a node, which is the number of edges between $v$ and the root of its tree. A root node has depth $d = 0$. 
For $v \in \Vcal_k$, define \eqn{\Wcal(v) = \{v, \Pcal(v), \Pcal^2(v),\dotsc,\Pcal^d(v)\}} as the set of waypoints from $v$ to the root of its tree, where $\Pcal : \Vcal_k \to \Vcal_k$ is the parent node of $v$. Let $\Vcal_{\rm near}$ be the set of all nodes within a ball of radius $R>0$ of the desired starting location of the backup trajectory, that is, \eqn{\Vcal_{\rm near} = \{v\in\Vcal_k ~|~ \norm{v - p^{\textup{nom}}_k(t_{kS})} \le R\}.} Then, for each $v \in \Vcal_{\rm near}$, we generate the corresponding trajectory $(p_v, u_v)$ by propagating the closed loop trajectory of tracking the waypoints~$\Wcal(v)$ over the horizon $[t_{kS},t_{kB}]$: \eqn{\label{eq:back_gen}\dot p_v(t) &= f(p_v(t),u_v(t)), ~p_v(t_{kS}) = p^{\textup{nom}}_k(t_{kS}), \\ u_v(t) &= \pi^T(p_v(t),\Wcal(v))} where $\pi^T : \Xcal \times 2^\Vcal_k \to \Ucal$ is a waypoint tracking controller\footnote{The path is not required to be tracked perfectly. Safety violations of the propagated trajectory is checked in the validation step.}, e.g. see~\cite{breivik2005guidance,amer2017modelling,rubi2020survey}. $2^\Vcal_k$ is the power set of $\Vcal_k$, i.e., the set of all sets of $\Vcal_k$. This forms a set $\Ncal_k^{T_S} \subset \Phi(t_{kS},p_k^{\rm nom}(t_{kS}))$ of $|\Vcal_{\rm near}|$ trajectories. Additionally, let $b_v^-$ be the budget state expenditure over the trajectory $p_v$ by solving \eqref{eq:budget_hybrid}. Finally, the trajectory with minimal $b_v^-$ that is safe for all time is selected as the backup trajectory: \eqn{\label{eq:back_select}p^{\textup{bck},T_S}_k = \argmin_{(p_v, u_v) \in \Ncal_k^{T_S}}\{b_v^-  ~|~p_v(t) \in \Fcal_k, ~\forall t \ge t_{kS}\}.}  The resulting candidate trajectory $p^{\textup{can},T_S}_k$ is validated by \Cref{def:valid}, i.e., checking that $b^-\le B$ by forward propagating \eqref{eq:budget_hybrid} along $p^{\textup{can},T_S}_k$. If there does not exist a feasible backup trajectory for a given switching time $T_S$, then the \gatekeeper{} algorithm decrements $T_S$ until either a feasible backup trajectory is found or $T_S = 0$ and the previous committed trajectory $p_{k-1}^{\rm com}$ is selected as the next committed trajectory.

In summary, we solve \Cref{prob:backup_prob} by selecting the backup trajectory of minimal budget expenditure among a set of potential backup paths in the \reroot{} forest. Through the forward propagation in \eqref{eq:back_gen} and the validation step in \eqref{eq:back_select}, we show that the backup trajectory is dynamically feasible and safe by construction. Finally, the \gatekeeper{} committed validation step checks budget constraint satisfaction of the candidate trajectory.

\section{Simulation Case Study}
\label{section:experiments}
We evaluate our method in a simulation case study of a fixed-wing UAV flying in a GNSS-denied environment, using visual-inertial odometry to estimate its position. The horizontal planar dynamics are modeled as a Dubins vehicle:
\eqn{\label{eq:dubins} \dot{r}_N=V\cos\psi,~\dot{r}_E=V\sin\psi,~\dot{\psi}=u} where $[r_N, r_E]^\top \in \R^2$ is the horizontal position in a North-East-Down (NED) coordinate frame, $\psi \in [0,2\pi)$ is the heading, $V>0$ is a fixed velocity and $u$ is the input turn rate. $u$ is constrained by the minimum turn radius of the UAV, $\rho>0$, such that $u \in \Ucal = [-\frac{V}{\rho}, \frac{V}{\rho}]$. For simplicity and visualization purposes, we constrain the UAV to fly at a fixed altitude, but our method extends trivially to 3D trajectories, or for more complicated dynamics.

The UAV's velocity is $V=10$~m/s and minimum turn radius is $\rho=10$~m, which corresponds to a $\sim$45\textdegree{} bank angle. The sensor \ac{FOV} is defined by a radius of 60 m and an angle of 90\textdegree. The perception and mapping module runs at 5 Hz. The objective is to reach the randomly chosen goal location within the 250~m $\times$ 200~m mission domain. The mission requirements are as follows: \begin{itemize}
    \item \textbf{(Safety constraint)} Maintain at least $N_f=8$ visual features in the \ac{FOV} between successive camera frames for robust feature tracking.
    \item \textbf{(Budget constraint)} Maintain an absolute position error (error bound on position) below $B=9$ m.
\end{itemize}
Note that the safety constraint is difficult to write in equation form, e.g., $h(x)\le0$, and is not differentiable. It is however easy to verify algorithmically.
We assume there exist some visual landmarks that reset position error to zero when the UAV can see them. We assume that the position error is proportional to distance traveled, with an error rate of 3\%, i.e, 0.03 m of error per 1 m traveled, based on~\cite{ellingson2018relative}.

We write the safe set as \eqn{\Scal = \{x \in \Xcal~|~ \textup{FeaturesInFoV}(x) \ge N_f \textup{ or } x \in \Rcal\},} where $\textup{FeaturesInFoV} : \Xcal \to \naturals$ is the number of visible features at state $x$. We assume no prior knowledge of the feature locations. The UAV maps the features as it flies through the environment and builds $\Fcal_k$. The budget state models the growth of localization error as \eqn{\begin{cases}\dot{b} = 0.03V, \quad &x \notin \Rcal \\ b^+ = b_{\textup{reset}}=0, \quad &x \in \Rcal.\end{cases}}

We define the budget renewal set to have subsets $\Rcal^i$ at the starting location, a mid-point landmark (unknown to the UAV at the mission start time), and around the known goal location. Since the UAV cannot hover at a single location under the dynamics \eqref{eq:dubins}, we define $\Rcal$ as a circular counter-clockwise orbit of radius $\rho$ around the landmarks. $\Rcal$ is controlled invariant under the constant control input $u = -\frac{V}{\rho}$ and reachable from nearby states, i.e., a backup set. The definition for $\Rcal^i$ is as follows:
\eqn{\Rcal^i = \Big\{x \in \Xcal~|~(r_N-l_N^i)^2+(r_E-l_E^i)^2=\rho^2,\\ \psi = \arctan \left(\frac{r_E-l_E^i}{r_N-l_N^i} \right) - \frac{\pi}{2} \Big \},} where $l_N^i$ and $l_E^i$ are the north and east positions of the $i$-th landmark. All backup trajectories must end in a subset $\Rcal^i$, which means the UAV can orbit indefinitely while not accumulating localization error. 
\begin{figure}[t]
    \centering
    \includegraphics[width=0.9\linewidth]{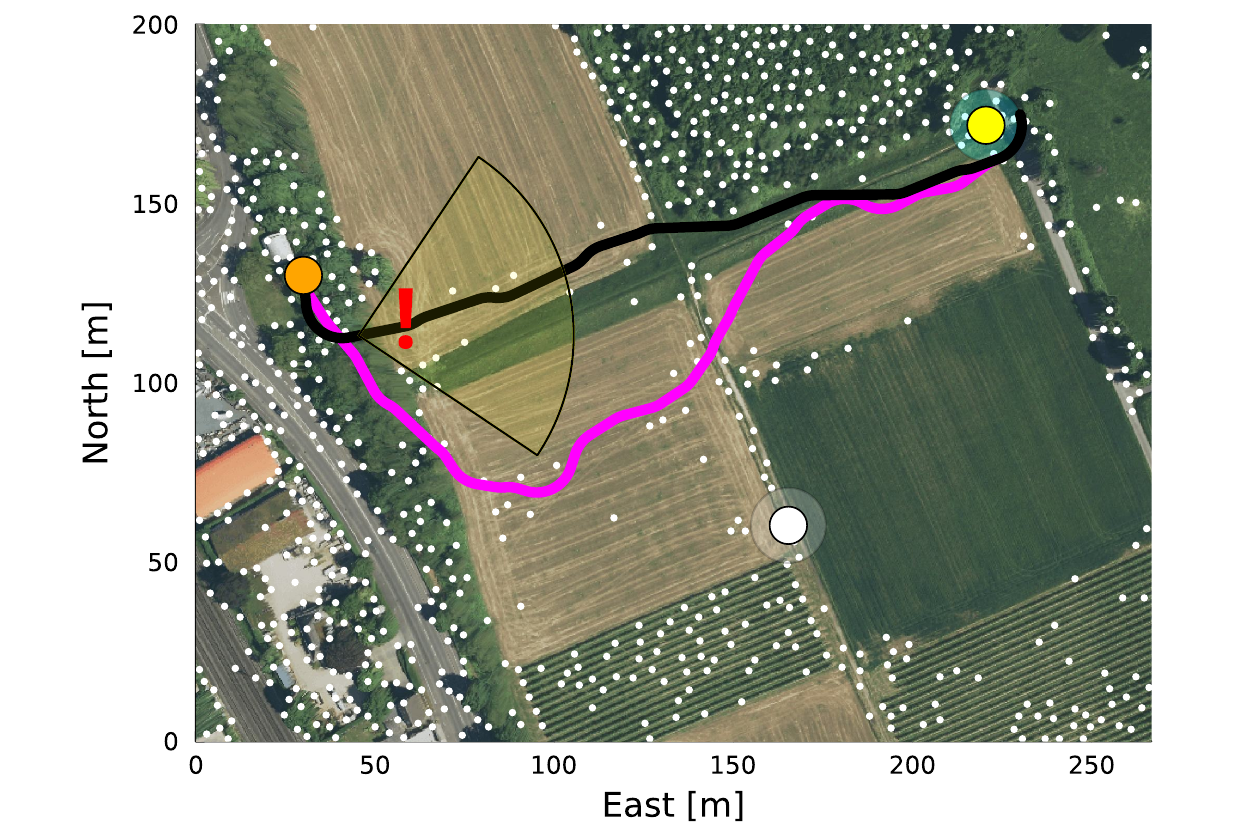}
    \caption{\small{Top-down view of field environment from the VPAIR database~\cite{schleiss2022vpair}. The white dots are visual odometry features. Note the lack of features in certain areas. The orange circle is the starting location. The yellow circle is the goal location. The white circle is a landmark. The nominal trajectory in black is the path of minimum distance, which becomes unsafe at the red ``!''. The magenta line is the \textit{omniscient} trajectory, i.e., has knowledge of all feature locations, that minimizes distance while satisfying the safety and budget constraints.}}
    \label{fig:omni}
    \vspace{-15pt}
\end{figure}
\begin{figure*}[t]
    \centering
    \includegraphics[width=\linewidth]{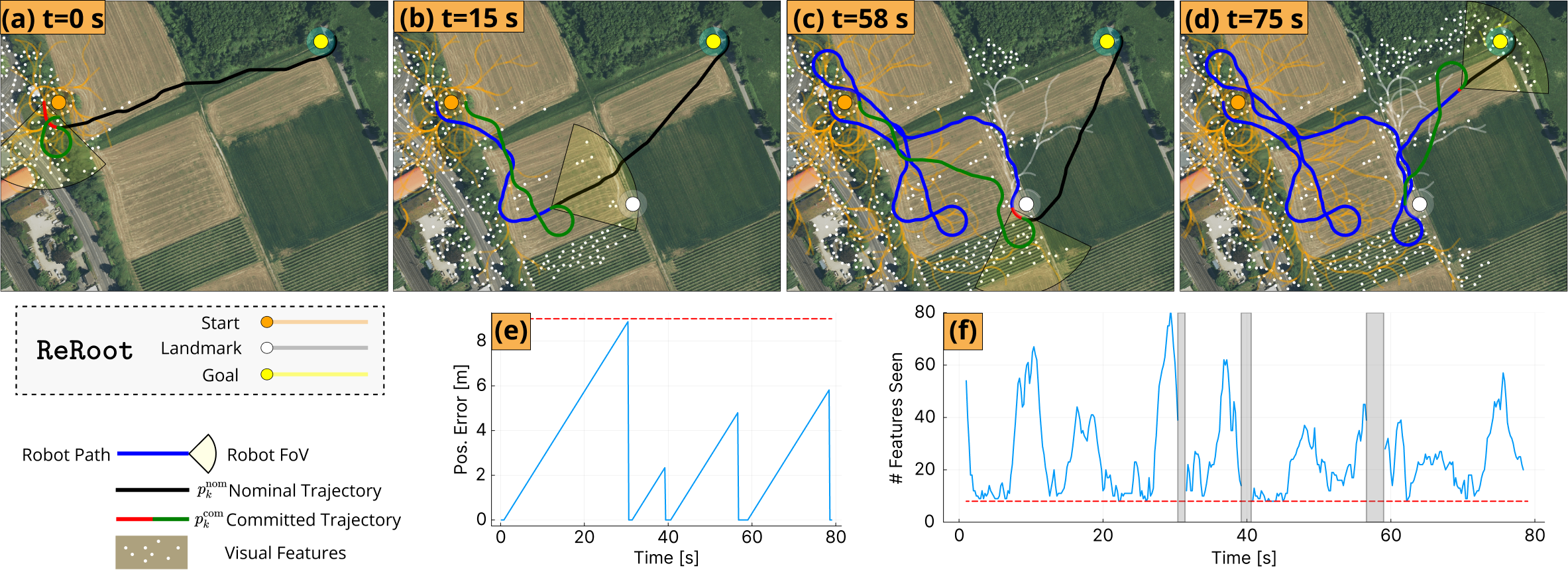}
    \caption{\small (a)-(d) Visualization of the simulated experiment of \gatekeeper{} with \reroot{} at various times in the mission. The features seen by the UAV are the white dots. The orange dot is the starting location, the gray dot is a landmark, and the yellow dot is the goal location. The colored thin lines are the branches of the \reroot{} forest. The blue line is the UAV's path up to time $t$. The black line is the unsafe part of the nominal trajectory, and the red line is the nominal component of the committed trajectory, which is replanned from the UAV position when the committed trajectory deviates from the nominal. The green line is the backup component of the committed trajectory and reaches a budget renewal set. (e) Value of the budget state $b$ (absolute position error) over time, which resets whenever the UAV reaches a budget renewal set. The budget never exceeds the maximum allowed value of 9 m. (f) Number of features in the \ac{FOV} over time, which is never below the minimum $N_f = 8$. The gray refers to when the UAV is at a landmark.}
    \label{fig:simulation}
    \vspace{-5pt}
\end{figure*}

\begin{figure*}
    \centering
    \includegraphics[width=\linewidth]{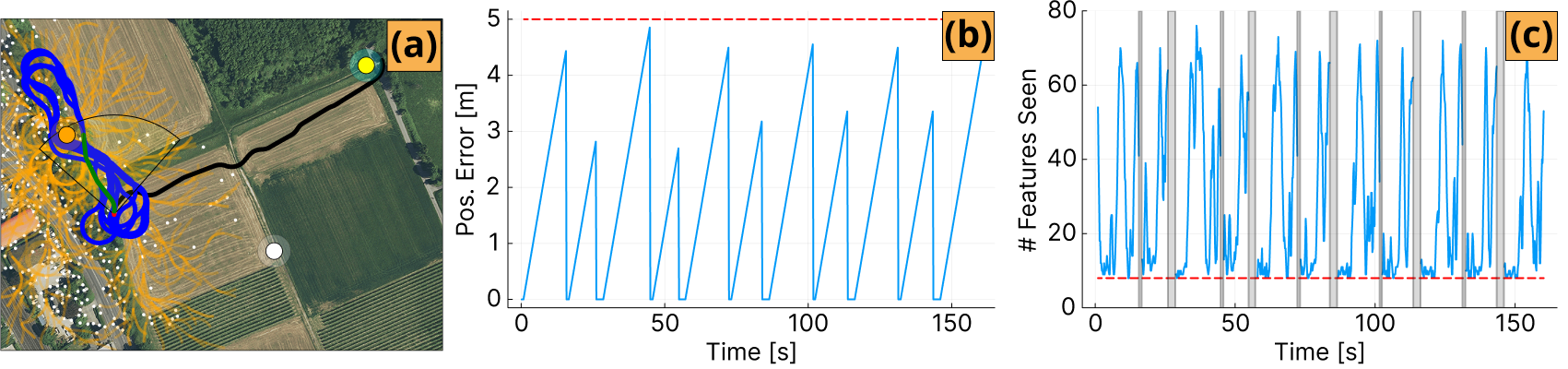}
    \caption{\small(a) Simulation result of the same setup as in \Cref{fig:simulation} but with a budget constraint of 5 m of localization error. (b) Despite never reaching the goal, the budget constraint is always satisfied and (c) the safety constraint is also always satisfied. A limit cycle-like behavior is observed, illustrating that \gatekeeper{} can successfully prevent the robot from leaving the safe set when the mission cannot be executed safely.}
    \label{fig:low_budget}
    \vspace{-15pt}
\end{figure*}

The simulated environment is a top-down image of a field taken from an airplane as part of the VPAIR database~\cite{schleiss2022vpair} and we create the set of features using Good Features to Track~\cite{shi1994good}, as shown in \Cref{fig:omni}. The mission domain is a 265 m $\times$ 200 m region. This environment presents a challenging setting for visual odometry, as certain regions of the field contain sparse or no discernible features. \gatekeeper{} with \reroot{} ensures that the UAV does not enter these unsafe regions. In the simulation with \gatekeeper{} and \reroot{}, the UAV has no knowledge of feature locations \textit{a priori} and maps them online when they enter the \ac{FOV}. The mid-field landmark is also unknown to the UAV and is discovered during the mission. We further discuss the trajectories in \Cref{fig:omni} in the results subsection below.

For the nominal planner, we use RRT* in a similar manner to the method in~\cite{ferguson2006replanning}, which facilitates replanning from any location in the map and returns a trajectory that is dynamically feasible for the Dubins dynamics \eqref{eq:dubins}.

We initialize \reroot{} around the starting location to form a small tree of safe backups that return to the initial orbit. Then, at 1 Hz, nodes are added based on the newly mapped features to extend the tree into the known safe set. As landmarks are discovered, a root node is placed and nearby nodes automatically connect to the new tree.

The Dubins dynamics in \eqref{eq:dubins} allows fast backup trajectory generation because the edges of \reroot{} represent dynamically feasible subpaths. Given two nodes $v_1,v_2 \in \Vcal_k$, the Dubins path of minimum distance is unique and is found algebraically \cite{shkel2001classification}. The path of waypoints $\Wcal(v)$ is dynamically feasible and the budget state $b(t)$ along the backup trajectory is the node cost of $v$, which is used when building \reroot{} to determine the parent of a new node. 

\Cref{fig:omni} depicts the nominal trajectory from the start to the goal, which becomes unsafe as there are fewer than $N_f=8$ features in the \ac{FOV}. We include for comparison the \textit{omniscient} trajectory that satisfies the safety and budget constraints in which the feature locations are known \textit{a priori}.

The code and animations of the simulation are available \href{https://github.com/dcherenson/budget-constrained-planning}{here}\footnote{https://github.com/dcherenson/budget-constrained-planning}. In the simulation, we observe that the UAV reaches the goal while satisfying all constraints in \Cref{fig:simulation}. \Cref{fig:simulation}a shows that the nominal trajectory across the sparse field is found to be unsafe, so a short backup trajectory is planned to return to its orbit. In \Cref{fig:simulation}b, the UAV has not mapped enough of the area around the mid-point landmark to reach it safely and must return to the starting location to reset the localization error. Then, in \Cref{fig:simulation}c, the UAV reaches the landmark using the \reroot{} tree rooted at the landmark. Finally, in \Cref{fig:simulation}d, the UAV is able to reach the goal location with a committed trajectory back to the landmark. \Cref{fig:simulation}e shows the position error over time which never exceeds the budget constraint. \Cref{fig:simulation}f shows the number of visible features over time, showing that the safety constraint is always satisfied. \Cref{tab:comp} shows the computation times of one iteration of \Cref{line:add_roots,line:grow_reroot} in \Cref{alg:autonomy} for \reroot{} and one iteration of \Cref{line:gk} for \gatekeeper{}.

\begin{table}[t]
    \centering
        \caption{Computation Times per Iteration}
        \label{tab:comp}
    \begin{tabular}{c|c|c} \hline
        \textbf{Component} & \textbf{Mean [ms] }& \textbf{Std. Dev. [ms]}  \\ \hline \hline
        \reroot{} & 2.90 & 7.17 \\ \hline
        \gatekeeper{} & 2.52 & 4.48 \\ \hline
    \end{tabular}
    \vspace{-20pt}
\end{table}

The behavior of the UAV trajectory in \Cref{fig:simulation} shows the exploratory effects of \gatekeeper{}. Since following the nominal trajectory to the goal is found to be unsafe, the UAV turns back towards the start. During the turn, it maps more features and is able to extend further into the newly mapped known safe set. Eventually, the budget constraint forces the return to the start to renew the accumulated error. Then, the nominal trajectory can be followed for a longer time as the known safe set has been mapped. This process continues until the UAV reaches the goal while mapping its environment. The shape of the resulting trajectory resembles the omniscient path in \Cref{fig:omni} with the addition of multiple return trips to landmarks.

We also simulate the mission with a lower, more challenging budget of $B=5$~m of accumulated localization error, which makes the task of reaching the goal while satisfying all constraints infeasible. In \Cref{fig:low_budget}a, the UAV is unable to reach the landmark or the goal without exceeding the limit. The resulting pattern is caused by the UAV following the nominal trajectory until \gatekeeper{} prevents the UAV from continuing to avoid exceeding the budget, then returning to the home landmark, and the cycle continues. \Cref{fig:low_budget}b and \Cref{fig:low_budget}c show that the safety and budget constraints are satisfied, despite this infinite looping behavior, which demonstrates the guaranteed constraint satisfaction.

\section{Conclusion}
In this paper, we proposed an architecture to guarantee safety and budget constraints throughout a mission in an environment where the safe set is built on-the-fly. 
The key contribution is \reroot{}, a sampling-based backup planning framework that augments the \gatekeeper{} architecture to construct backup policies, which are needed for guaranteeing constraint satisfaction. By growing multiple reverse RRT* trees rooted in renewal sets, \reroot{} efficiently generates trajectories that satisfy constraints while minimizing resource expenditure.
The efficacy of our approach was demonstrated in simulation with a case study of a fixed-wing UAV in a GNSS-denied environment navigating to a goal location. The safety constraint was to maintain a minimum of 8~visual features in the \ac{FOV} and the budget constraint was to limit the localization error to below 9~m. \gatekeeper{} with \reroot{} allowed the UAV to explore and map features while always satisfying the budget and safety constraints.

Future work could include extending our method of backup trajectory construction to time-varying safe sets, e.g., avoiding a dynamic obstacle, and time-varying budget renewal sets, e.g., a mobile charging station. Also of interest would be to incorporate multiple budget constraints into \reroot{} and form a set of Pareto optimal backup trajectories. The active budget constraint in the \gatekeeper{} iteration could then determine which backup to take.
{
\setstretch{0.95}
\bibliographystyle{IEEEtran}
\bibliography{biblio, IEEEabrv}
}

\end{document}

%% file: preamble.tex
\usepackage{amsmath,amssymb,amsfonts}
\usepackage{bbm}
\usepackage{algorithmic}
\usepackage{graphicx}
\usepackage{textcomp}
\usepackage{setspace}
\usepackage{booktabs}
\usepackage{comment}
\usepackage{tabularx}

\usepackage{epsfig}
\usepackage{times} 
\usepackage[inkscapelatex=false]{svg}
\usepackage[linesnumbered,ruled,vlined]{algorithm2e} 

\usepackage[nolist, nohyperlinks]{acronym}
\acrodef{GNC}[GNC]{Graduated-Nonconvexity}
\acrodef{WLS}[WLS]{Weighted Least Squares}
\acrodef{TLS}[TLS]{Truncated Least Squares}
\acrodef{SVD}[SVD]{Singular Value Decomposition}
\acrodef{VO}[VO]{Visual Odometry}
\acrodef{VIO}[VIO]{Visual-Inertial Odometry}
\acrodef{SDF}[SDF]{Signed Distance Field}
\acrodef{SDFs}[SDFs]{Signed Distance Fields}
\acrodef{ESDF}[ESDF]{Euclidean Signed Distance Field}
\acrodef{CESDF}[C-ESDF]{Certified \ac{ESDF}}
\acrodef{TSDF}[TSDF]{Truncated Signed Distance Field}
\acrodef{RGBD}[RGBD]{RGB-Depth}
\acrodef{IMU}[IMU]{Inertial Measurement Unit}
\acrodef{EKF}[EKF]{Extended Kalman Filter}
\acrodef{VO}[VO]{Visual Odometry}
\acrodef{CVO}[C-VO]{Certified Visual Odometry}
\acrodef{ISS}[ISS]{Input-to-State}
\acrodef{SLAM}[SLAM]{Simultaneous Localization and Mapping}
\acrodef{SFC}[SFC]{Safe Flight Corridors}
\acrodef{FOV}[FoV]{field of view}
\acrodef{UAV}[UAV]{}

\usepackage{cite}

\usepackage{bm}
\usepackage{breqn}
\usepackage{amsmath}
\usepackage{amssymb}
\usepackage{tabularx}

\usepackage{amsthm}

\newcommand{\naturals}{\mathbb{N}}

\newcommand{\reals}{\mathbb{R}}
\newcommand{\R}{\reals}
\newcommand{\Rnonneg}{\reals_{\geq 0}}

\newcommand{\Bcal}{\mathcal{B}}
\newcommand{\Ccal}{\mathcal{C}}

\newcommand{\Ecal}{\mathcal{E}}
\newcommand{\Fcal}{\mathcal{F}}
\newcommand{\Gcal}{\mathcal{G}}

\newcommand{\Ical}{\mathcal{I}}

\newcommand{\Ncal}{\mathcal{N}}

\newcommand{\Pcal}{\mathcal{P}}

\newcommand{\Rcal}{\mathcal{R}}
\newcommand{\Scal}{\mathcal{S}}
\newcommand{\Tcal}{\mathcal{T}}
\newcommand{\Ucal}{\mathcal{U}}
\newcommand{\Vcal}{\mathcal{V}}
\newcommand{\Wcal}{\mathcal{W}}
\newcommand{\Xcal}{\mathcal{X}}

\newcommand{\eqn}[1]{\begin{align} #1 \end{align}}
\newcommand{\eqnN}[1]{\begin{align*} #1 \end{align*}}

\newcommand{\norm}[1]{\left\Vert #1 \right \Vert}

\newcommand{\gatekeeper}{\texttt{gatekeeper}}
\newcommand{\reroot}{\texttt{ReRoot}}

\theoremstyle{plain}
\newtheorem{theorem}{Theorem}

\newtheorem{problem}{Problem}
\newtheorem{definition}{Definition}

\theoremstyle{definition}
\newtheorem{assumption}{Assumption}
\newtheorem{remark}{Remark}

\DeclareMathOperator*{\argmin}{\arg\!\min}

\theoremstyle{remark}

\makeatletter
\let\NAT@parse\undefined
\makeatother
\usepackage{hyperref}
\hypersetup{
colorlinks, 
    linkcolor={red!50!black},
    citecolor={blue!50!black},
    urlcolor={blue!80!black}
}
\usepackage{cleveref}

\crefformat{problem}{Problem~#2#1#3}
\crefformat{assumption}{Assumption~#2#1#3}

%% file: root.bbl
\begin{thebibliography}{10}
\providecommand{\url}[1]{#1}
\csname url@rmstyle\endcsname
\providecommand{\newblock}{\relax}
\providecommand{\bibinfo}[2]{#2}
\providecommand\BIBentrySTDinterwordspacing{\spaceskip=0pt\relax}
\providecommand\BIBentryALTinterwordstretchfactor{4}
\providecommand\BIBentryALTinterwordspacing{\spaceskip=\fontdimen2\font plus
\BIBentryALTinterwordstretchfactor\fontdimen3\font minus \fontdimen4\font\relax}
\providecommand\BIBforeignlanguage[2]{{%
\expandafter\ifx\csname l@#1\endcsname\relax
\typeout{** WARNING: IEEEtran.bst: No hyphenation pattern has been}%
\typeout{** loaded for the language `#1'. Using the pattern for}%
\typeout{** the default language instead.}%
\else
\language=\csname l@#1\endcsname
\fi
#2}}

\bibitem{kumar2010efficient}
A.~Kumar and A.~Vladimirsky, ``An efficient method for multiobjective optimal control and optimal control subject to integral constraints,'' \emph{Journal of Computational Mathematics}, pp. 517--551, 2010.

\bibitem{takei2015optimal}
R.~Takei, W.~Chen, Z.~Clawson, S.~Kirov, and A.~Vladimirsky, ``Optimal control with budget constraints and resets,'' \emph{SIAM Journal on Control and Optimization}, vol.~53, no.~2, pp. 712--744, 2015.

\bibitem{tsiogkas2018dcop}
N.~Tsiogkas and D.~M. Lane, ``Dcop: Dubins correlated orienteering problem optimizing sensing missions of a nonholonomic vehicle under budget constraints,'' \emph{IEEE Robotics and Automation Letters}, vol.~3, no.~4, pp. 2926--2933, 2018.

\bibitem{yang2021uav}
Y.~Yang, J.~Khalife, J.~J. Morales, and Z.~M. Kassas, ``Uav waypoint opportunistic navigation in \textrm{GNSS}-denied environments,'' \emph{IEEE Transactions on Aerospace and Electronic Systems}, vol.~58, no.~1, pp. 663--678, 2021.

\bibitem{naveed2024eclares}
K.~B. Naveed, D.~Agrawal, C.~Vermillion, and D.~Panagou, ``Eclares: Energy-aware clarity-driven ergodic search,'' in \emph{IEEE International Conference on Robotics and Automation}, 2024, pp. 14\,326--14\,332.

\bibitem{bopardikar2014multi}
S.~D. Bopardikar, B.~Englot, and A.~Speranzon, ``Multi-objective path planning in \textrm{GPS} denied environments under localization constraints,'' in \emph{American Control Conference (ACC)}, 2014, pp. 1872--1879.

\bibitem{gilles2020evasive}
M.~A. Gilles and A.~Vladimirsky, ``Evasive path planning under surveillance uncertainty,'' \emph{Dynamic Games and Applications}, vol.~10, no.~2, pp. 391--416, 2020.

\bibitem{lin2023safe}
Q.~Lin, B.~Tang, Z.~Wu, C.~Yu, S.~Mao, Q.~Xie, X.~Wang, and D.~Wang, ``Safe offline reinforcement learning with real-time budget constraints,'' in \emph{International Conference on Machine Learning (ICML)}, 2023, pp. 21\,127--21\,152.

\bibitem{notomista2018persistification}
G.~Notomista, S.~F. Ruf, and M.~Egerstedt, ``Persistification of robotic tasks using control barrier functions,'' \emph{IEEE Robotics and Automation Letters}, vol.~3, no.~2, pp. 758--763, 2018.

\bibitem{hobbs2023runtime}
K.~L. Hobbs, M.~L. Mote, M.~C. Abate, S.~D. Coogan, and E.~M. Feron, ``Runtime assurance for safety-critical systems: An introduction to safety filtering approaches for complex control systems,'' \emph{IEEE Control Systems Magazine}, vol.~43, no.~2, pp. 28--65, 2023.

\bibitem{kim2021backup}
H.~Kim, H.~Yoon, W.~Wan, N.~Hovakimyan, L.~Sha, and P.~Voulgaris, ``Backup plan constrained model predictive control,'' in \emph{IEEE Conference on Decision and Control (CDC)}.\hskip 1em plus 0.5em minus 0.4em\relax IEEE, 2021, pp. 289--294.

\bibitem{thananjeyan2021recovery}
B.~Thananjeyan, A.~Balakrishna, S.~Nair, M.~Luo, K.~Srinivasan, M.~Hwang, J.~E. Gonzalez, J.~Ibarz, C.~Finn, and K.~Goldberg, ``Recovery \textrm{RL}: Safe reinforcement learning with learned recovery zones,'' \emph{IEEE Robotics and Automation Letters}, vol.~6, no.~3, pp. 4915--4922, 2021.

\bibitem{kiemel2024safe}
J.~Kiemel, L.~Righetti, T.~Kr{\"o}ger, and T.~Asfour, ``Safe reinforcement learning of robot trajectories in the presence of moving obstacles,'' \emph{IEEE Robotics and Automation Letters}, 2024.

\bibitem{jung2025contigency}
L.~Jung, A.~Estornell, and M.~Everett, ``Contingency constrained planning with \textrm{MPPI} within \textrm{MPPI},'' in \emph{Learning for Dynamics and Control (L4DC)}, 2025.

\bibitem{blanchini1999set}
F.~Blanchini, ``Set invariance in control,'' \emph{Automatica}, vol.~35, no.~11, pp. 1747--1767, 1999.

\bibitem{agrawal2024gatekeeper}
D.~R. Agrawal, R.~Chen, and D.~Panagou, ``gatekeeper: Online safety verification and control for nonlinear systems in dynamic environments,'' \emph{IEEE Transactions on Robotics}, vol.~40, pp. 4358--4375, 2024.

\bibitem{agrawal2025online}
D.~R. Agrawal and D.~Panagou, ``Online safety under multiple constraints and input bounds using gatekeeper: Theory and applications,'' \emph{IEEE Control Systems Letters}, 2025.

\bibitem{karaman2011sampling}
S.~Karaman and E.~Frazzoli, ``Sampling-based algorithms for optimal motion planning,'' \emph{The International Journal of Robotics Research}, vol.~30, no.~7, pp. 846--894, 2011.

\bibitem{perez2012lqr}
A.~Perez, R.~Platt, G.~Konidaris, L.~Kaelbling, and T.~Lozano-Perez, ``\textrm{LQR-RRT*}: Optimal sampling-based motion planning with automatically derived extension heuristics,'' in \emph{IEEE International Conference on Robotics and Automation (ICRA)}, 2012, pp. 2537--2542.

\bibitem{webb2013kinodynamic}
D.~J. Webb and J.~Van Den~Berg, ``Kinodynamic \textrm{RRT}*: Asymptotically optimal motion planning for robots with linear dynamics,'' in \emph{IEEE International Conference on Robotics and Automation (ICRA)}, 2013, pp. 5054--5061.

\bibitem{kuffner2000rrt}
J.~J. Kuffner and S.~M. LaValle, ``\textrm{RRT}-connect: An efficient approach to single-query path planning,'' in \emph{IEEE International Conference on Robotics and Automation (ICRA)}, 2000, pp. 995--1001.

\bibitem{jordan2013optimal}
M.~Jordan and A.~Perez, ``Optimal bidirectional rapidly-exploring random trees,'' 2013.

\bibitem{qureshi2015intelligent}
A.~H. Qureshi and Y.~Ayaz, ``Intelligent bidirectional rapidly-exploring random trees for optimal motion planning in complex cluttered environments,'' \emph{Robotics and Autonomous Systems}, vol.~68, pp. 1--11, 2015.

\bibitem{breivik2005guidance}
M.~Breivik and T.~I. Fossen, ``Guidance-based path following for autonomous underwater vehicles,'' in \emph{MTS/IEEE OCEANS}, 2005, pp. 2807--2814.

\bibitem{amer2017modelling}
N.~H. Amer, H.~Zamzuri, K.~Hudha, and Z.~A. Kadir, ``Modelling and control strategies in path tracking control for autonomous ground vehicles: A review of state of the art and challenges,'' \emph{Journal of intelligent \& robotic systems}, vol.~86, no.~2, pp. 225--254, 2017.

\bibitem{rubi2020survey}
B.~Rub{\'\i}, R.~P{\'e}rez, and B.~Morcego, ``A survey of path following control strategies for uavs focused on quadrotors,'' \emph{Journal of Intelligent \& Robotic Systems}, vol.~98, no.~2, pp. 241--265, 2020.

\bibitem{ellingson2018relative}
G.~Ellingson, K.~Brink, and T.~McLain, ``Relative visual-inertial odometry for fixed-wing aircraft in \textrm{GPS}-denied environments,'' in \emph{IEEE/ION Position, Location and Navigation Symposium (PLANS)}, 2018, pp. 786--792.

\bibitem{schleiss2022vpair}
M.~Schleiss, F.~Rouatbi, and D.~Cremers, ``Vpair-aerial visual place recognition and localization in large-scale outdoor environments,'' \emph{arXiv preprint arXiv:2205.11567}, 2022.

\bibitem{shi1994good}
J.~Shi and C.~Tomasi, ``Good features to track,'' in \emph{IEEE Conference on Computer Vision and Pattern Recognition (CVPR)}, 1994, pp. 593--600.

\bibitem{ferguson2006replanning}
D.~Ferguson, N.~Kalra, and A.~Stentz, ``Replanning with \textrm{RRT}s,'' in \emph{IEEE International Conference on Robotics and Automation (ICRA)}, 2006, pp. 1243--1248.

\bibitem{shkel2001classification}
A.~M. Shkel and V.~Lumelsky, ``Classification of the dubins set,'' \emph{Robotics and Autonomous Systems}, vol.~34, no.~4, pp. 179--202, 2001.

\end{thebibliography}
